\documentclass{article}

\usepackage{microtype}
\usepackage{graphicx}
\usepackage{subcaption}
\usepackage{booktabs} % for professional tables

\usepackage{hyperref}

 \usepackage[preprint]{icml2026}

\usepackage[table]{xcolor}
\usepackage{amsmath}
\usepackage{amssymb}
\usepackage{mathtools}
\usepackage{amsthm}
\usepackage{longtable}
\usepackage{booktabs} 
\usepackage{multirow}
\usepackage{algorithm}
\usepackage{algorithmic}

\usepackage[capitalize,noabbrev]{cleveref}

\theoremstyle{plain}
\newtheorem{theorem}{Theorem}[section]

\theoremstyle{definition}
\newtheorem{definition}[theorem]{Definition}

\theoremstyle{remark}

\usepackage[textsize=tiny]{todonotes}

\icmltitlerunning{Stabilized Fine-Tuning with LoRA via Scaling Factor}

\begin{document}

\twocolumn[
  \icmltitle{\texorpdfstring{Stabilized Fine-Tuning with LoRA in Federated Learning: \\ Mitigating the Side Effect of Client Size and Rank via the Scaling Factor}{Stabilized Fine-Tuning with LoRA in Federated Learning: Mitigating the Side Effect of Client Size and Rank via the Scaling Factor}}



% 如果需要标注等效贡献（Equal Contribution），可以保留此行并在 author 后面加 {equal,bupt}
    
    \begin{icmlauthorlist}
        \icmlauthor{Jiayu Huang}{bupt}
        \icmlauthor{Xiaohu Wu}{bupt}
        \icmlauthor{Tiantian He}{astar}
        \icmlauthor{Qicheng Lao}{bupt}
    \end{icmlauthorlist}
    
    % 机构定义（已删除 Mila）
    \icmlaffiliation{bupt}{Beijing University of Posts and Telecommunications, Beijing, China}
    \icmlaffiliation{astar}{Agency for Science, Technology and Research, Singapore, Singapore}
    
    % 设置 Xiaohu Wu 为通讯作者
    \icmlcorrespondingauthor{Xiaohu Wu}{xiaohu.wu@bupt.edu.cn}

  \icmlkeywords{Machine Learning, ICML}

  \vskip 0.3in
]



\printAffiliationsAndNotice{}  % no special notice (required even if empty)

\begin{abstract}
    Large Language Models (LLMs) are pivotal in natural language processing. The impracticality of full fine-tuning has prompted Parameter-Efficient Fine-Tuning (PEFT) methods like Low-Rank Adaptation (LoRA), optimizing low-rank matrices $A$ and $B$. In distributed scenarios where privacy constraints necessitate Federated Learning (FL), however, the integration of LoRA is often unstable. Specifically, we identify that aggregating updates from multiple clients introduces statistical variance that scales with the client count, causing gradient collapse when using high-rank adapters. Existing scaling factor candidates, such as the one used by Rank-Stabilized LoRA, ignore the interaction caused by the aggregation process. To bridge this gap, this paper introduces Stabilized Federated LoRA (SFed-LoRA), a framework that theoretically characterizes the interaction between adapter rank and federated aggregation. We derive an optimal scaling factor, $\gamma_z = \alpha\sqrt{N/r}$, designed to effectively mitigate the aggregation error accumulating across $N$ clients. By correcting the scaling mismatch inherent in previous approaches, SFed-LoRA restores the efficacy of high-rank adaptation without altering the original model architecture or increasing inference latency. Extensive experiments in diverse tasks, model architectures, and heterogeneous data distributions are conducted to validate our results. We demonstrate that SFed-LoRA prevents high-rank collapse, and achieves significantly improved stability and faster convergence compared with state-of-the-art baselines for high-rank adaptation. Code is available in the Appendix.
\end{abstract}

\begin{figure*}[htbp]
    \centering
    \includegraphics[width=0.9\textwidth]{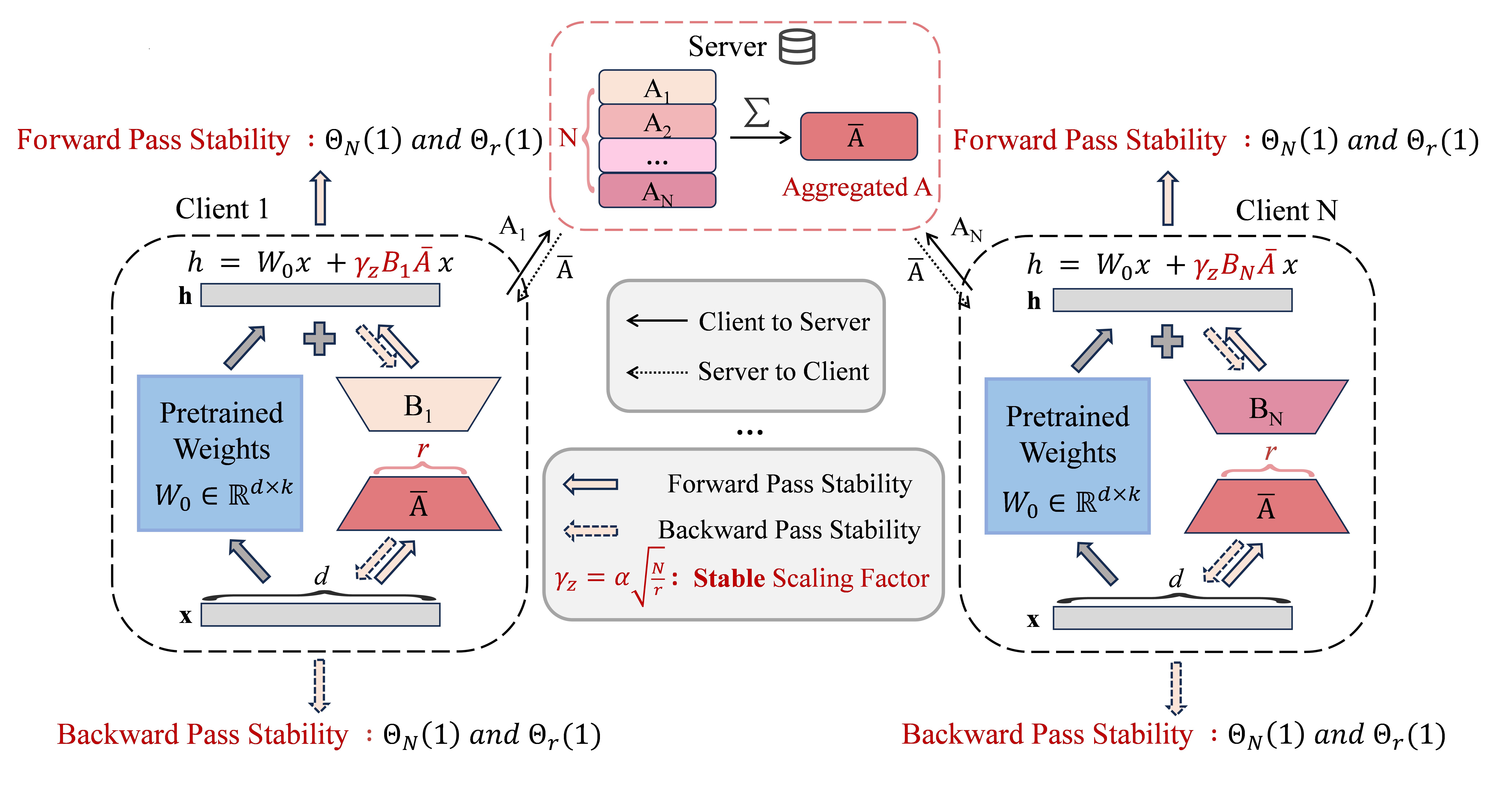} % Adjusted width to avoid overfull boxes
    \caption{\textbf{The framework of SFed-LoRA.} It adopts a split aggregation strategy where clients upload only matrix $A$ while maintaining $B$ locally to protect privacy. A novel scaling factor $\gamma_z = \alpha\sqrt{N/r}$ is integrated into the local computation (as shown in the equation) to counteract aggregation interference, ensuring stable training performance across varying client numbers.}
    \label{fig1}
\end{figure*}

\section{Introduction}

% Paragraph 1: Background (Establishes 'r')
Large Language Models (LLMs), often referred to as Pre-trained Language Models (PLMs), have become central to natural language processing but necessitate efficient adaptation strategies to align with specific downstream tasks~\citep{llama,gpt-4}. Full fine-tuning provides the largest adaptation scope; however, its computational and memory demands are prohibitively high, hindering deployment of such models in many practical applications~\citep{parameter}. To mitigate this, Low-Rank Adaptation (LoRA) has emerged as a premier Parameter-Efficient Fine-Tuning (PEFT) method. By freezing the pre-trained backbone and optimizing the product of two low-rank matrices \(A \in \mathbb{R}^{r \times k}\) and \(B \in \mathbb{R}^{d \times r}\) (where the rank \(r \ll \min(d, k)\)), LoRA significantly reduces the parameter budget while preserving model performance~\citep{lora}. Comprehensive model fine-tuning generally requires massive centralized datasets, which may not translate into feasible real-world deployments, as most large language model training datasets are dispersed among institutes under strict data privacy constraints. Federated Learning (FL) resolves this tension by enabling collaborative, decentralized model training across organizations without requiring raw data exchange~\citep{mcmahan}. Recognizing this synergy, recent frameworks such as FedIT~\citep{fedit}, FFA-LoRA~\citep{ffa}, and FedSA-LoRA~\citep{fedsa} have demonstrated promising progress in integrating LoRA with federated protocols.

% Paragraph 2: Problem Statement (Establishes 'alpha' and 'N')
Despite these advancements, a critical instability persists in LoRA-based federated fine-tuning, highlighting a significant theoretical deficiency in current research. Traditionally, LoRA formulates an adaptation scaling factor \(\gamma = \alpha/r\), using a constant hyperparameter \(\alpha\) to regulate the amount of adapted updates, resulting in gradient-collapse at high rank and restricting the effectiveness of adaptation to low-rank ones. While Rank-stabilized LoRA (rsLoRA) addresses this in standalone settings via \(\gamma_r = \alpha/\sqrt{r}\)~\citep{rslora}—corroborating the assertion by \citet{parameter} that larger ranks should yield superior performance—it fails to account for the federated context. Specifically, the aggregation process involves summing updates from multiple clients, introducing statistical variance that inherently scales with the participating client count \(N\). This accumulation of variance creates interference that compromises the stability established by rsLoRA. To systematically investigate this phenomenon in the backdrop of the proliferation of FL-LoRA frameworks~\citep{wu2025survey,ren2025advances}, and address our fundamental questions about LoRA in the federated setting, we settle on the framework of FedSA-LoRA~\citep{fedsa} due to its design choice: it aggregates only the \(A\) matrices while excluding \(B\), effectively isolating the source of aggregation error and supporting a precise examination of how the scaling factor influences stability within federated settings.

% Paragraph 3: The Proposal (Synthesizes r, alpha, N)
In this paper, we propose \textbf{Stabilized Federated LoRA (SFed-LoRA)}, a novel framework designed to counteract this aggregation-induced instability and unleash the full potential of the high-rank adaptation in federated learning. SFed-LoRA precisely captures the specific role of the aggregating matrix \(A\) in training dynamics, and obtains a theoretical optimal scaling factor \(\gamma_z = \alpha\sqrt{N/r}\), where \(z = (N, r)\). This formulation extends the principles of rsLoRA by explicitly integrating the client count \(N\) to normalize the variance accumulation identified above. As depicted in Figure~\ref{fig1}, SFed-LoRA integrates this scaling mechanism while preserving the original LoRA adapter structure, thereby incurring no additional inference latency. Comprehensive experiments on benchmark datasets, including GSM8K~\citep{gsm8k} and GLUE~\citep{glue}, using models such as LLaMA 2~\citep{llama} and RoBERTa-large~\citep{roberta}, reveal that SFed-LoRA can remarkably suppress gradient collapse and aid stable training at higher ranks. Consequently, it notably outperforms both standard LoRA and rsLoRA in terms of stability and convergence rate, leveraging \(\gamma_z\) to robustly mitigate aggregation effects.

% Paragraph 4: Contributions
Our contributions are summarized as follows:
\begin{itemize}
    \item We provide a theoretical derivation proving that \(\gamma_z = \alpha\sqrt{N/r}\) constitutes the optimal scaling factor for ensuring rank stability and consistent gradient norms in federated LoRA fine-tuning.
    \item We design and introduce SFed-LoRA, which successfully alleviates the adverse effects of federated aggregation and overcomes performance difficulties when we lower the rank to guarantee stable and efficient training.
    \item We present extensive experimental evidence, demonstrating SFed-LoRA's outstanding advantages over both standard LoRA and rsLoRA, corroborating its efficacy for LLM fine-tuning in various federated-learning scenarios.
\end{itemize}

\section{Background and Related Works}

We firstly overview the concept of fine-tuning with LoRA and its combination into federated learning, elucidating the significance of the scaling factor in addressing stability challenges. Then we recap the analytical methods in this work. We base our study on established frameworks \cite{yang2020feature,rslora}, which provide the theoretical underpinnings for our investigations.

\subsection{Federated LoRA (Fed-LoRA)}

\subsubsection{LoRA}

LoRA freezes pre-trained model weights and injects trainable low-rank matrices into each layer. Specifically, regarding the weight matrix $W_0 \in \mathbb{R}^{d \times k}$, LoRA introduces $A \in \mathbb{R}^{r \times k}$ and $B \in \mathbb{R}^{d \times r}$ (where $r \ll \min(d, k)$), reparameterizing the forward pass as:
\begin{equation}
h = W_0 x + \gamma BAx,
\end{equation}
where $A$ is initialized with i.i.d. Gaussian entries and $B$ as a zero matrix. The scaling factor $\gamma$, conventionally set to $\alpha/r$, aims to normalize the update magnitude across ranks. While LoRA improves parameter efficiency, its standard scaling induces gradient instability at higher ranks in federated settings—a limitation SFed-LoRA addresses through a theoretically derived scaling factor.

\subsubsection{Fed-LoRA}

Fed-LoRA integrates Low-Rank Adaptation (LoRA) into federated learning. This paradigm facilitates efficient and private fine-tuning of LLM on dispersed clients. Federated learning was pioneered by \citet{mcmahan} through FedAvg, promoting collaborative training without sharing raw data. The initial application of LoRA in this context was advanced by \citet{fedit}, who proposed FedIT, integrating LoRA with FedAvg for instruction tuning by aggregating both $A$ and $B$ matrices.

Subsequently, numerous works have enhanced this initial design. \citet{ffa} introduced FFA-LoRA, which performs the same architecture as FedIT but freezes $A$ matrices to decrease the communication overhead. To diminish the constraint on expressiveness limitations brought by the freezing projections, \citet{rolora} proposed RoLoRA, employing alternating optimization of $A$ and $B$ to alleviate the interference caused by aggregation. For heterogeneous settings, \citet{flora} presented FLoRA with adaptive stacked LoRA blocks, while \citet{flex} proposed FlexLoRA to dynamically rescale local ranks. Additionally, \citet{duap} developed FedDPA, utilizing dual adapters to handle test-time distribution shifts. These works have significantly advanced federated LoRA methods.

\citet{fedsa} introduced FedSA-LoRA, which only aggregates $A$ matrices but keeps $B$ as a local adaptive matrix to eliminate FedIT's unavoidable aggregation errors. This design enhances performance over FFA-LoRA due to the freedom from the fixed initialization constraint. These features lay out a solid foundation for our SFed-LoRA, which is an extension of FedSA-LoRA with a novel scaling factor $\gamma_z$ to further improve rank stability in federated settings.

\subsubsection{The Role of the Scaling Factor}

The scaling factor \(\gamma\) modulates the magnitude of the matrix product \(BA\), serving as a cornerstone for training stability. The original LoRA formulation \citep{lora} sets \(\gamma = \alpha/r\). While effective for low ranks, this scaling excessively dampens updates as \(r\) increases, causing gradient collapse at higher ranks.

\citet{rslora} addressed this by introducing an evolved factor \(\gamma_r = \alpha/\sqrt{r}\) in rsLoRA that standardizes magnitudes while restoring the performance for high-rank standalone training. However, this derivation ignores the weight aggregation process under Federated Learning. As client updates are averaged, the variance of the aggregated adapter, therefore, shifts across clients and the single-client scaling factor is no longer optimal in the distributed setting.

SFed-LoRA fills this gap by proposing a novel federated scaling factor \(\gamma_z = \alpha\sqrt{\frac{N}{r}}\) that theoretically incorporates client count \(N\), which in theory offsets the reduced variance induced by aggregation and guarantees the consistency of the gradient norms across clients and ranks, allowing for high-rank LLM fine-tuning under the federated setting.

\subsection{The Used Analytical Method}

Our theoretical analysis is grounded in the feature learning framework of the infinite-width limit \cite{yang2020feature}, which evaluates the impact of parameterization on model stability. \cite{yang2020feature} introduced the $abc$-parametrization to characterize learning dynamics in fully-connected networks, defining weights as $W^l = d^{-a_l} W^l_{i,j}$ with entries $W^l_{i,j} \sim \mathcal{N}(0, d^{-b_l})$ and learning rate $\eta d^{-c}$. Their analysis proves that standard schemes (where $a_l = c = 0$) fail to sustain stable feature learning in high dimensions. Instead, they derive stable configurations (e.g., $b_l = 1, c = 0$) that ensure non-collapsing training dynamics.

\citet{rslora} adapted this trajectory analysis to LoRA, optimizing the scaling factor to stabilize rank variations in standalone training. We extend this methodology to Federated Learning by incorporating aggregation dynamics into the learning trajectory analysis. This allows us to derive the federated-optimal scaling factor $\gamma_z = \alpha\sqrt{N/r}$, ensuring that SFed-LoRA's stability guarantees are rigorously rooted in established infinite-width theory.

\section{The FedSA-LoRA Framework}

Consider a federated learning system comprising a central server and $N$ clients. Each client $i$ possesses a private dataset and maintains two low-rank adapter matrices: $A_{i}^{(t)} \in \mathbb{R}^{r \times k}$ and $B_{i}^{(t)} \in \mathbb{R}^{d \times r}$.

To preserve the pre-trained model's capabilities at the onset, initialization follows the standard LoRA protocol:
\begin{itemize}
    \item $B_{i}^{(0)}$ is initialized as a zero matrix.
    \item $A_{i}^{(0)}$ is initialized with random Gaussian entries $a_{pq} \sim \mathcal{N}(0, \sigma^2)$.
\end{itemize}
Consequently, $\Delta W_{i}^{(0)} = B_{i}^{(0)} A_{i}^{(0)} = 0$, ensuring the initial model $W_{i}^{(0)} = W_0 + \Delta W_{i}^{(0)}$ remains identical to the pre-trained weights $W_0$.

FedSA-LoRA \citep{fedsa} is predicated on the hypothesis that the down-projection $A_{i}$ encodes \emph{general knowledge}, whereas the up-projection $B_{i}$ captures \emph{client-specific knowledge}. Accordingly, the framework splits the aggregation process during the $t$-th round:

\begin{enumerate}
    \item \textbf{Local Training:} Each client $i$ updates both $A_i^{(t)}$ and $B_i^{(t)}$ on their local dataset.
    \item \textbf{Selective Upload:} Clients upload only $A_i^{(t)}$ to the server, keeping $B_i^{(t)}$ local.
    \item \textbf{Aggregation:} The server broadcasts the global average $\bar{A}^{(t+1)} = \frac{1}{N} \sum_{i=1}^N A_i^{(t)}$.
    \item \textbf{Local Update:} Clients update their model as $W_i^{(t+1)} = W_0 + B_i^{(t)} \bar{A}^{(t+1)}$.
\end{enumerate}

This selective strategy offers a critical theoretical advantage. Since the average of matrix products generally deviates from the product of averages (i.e., $\frac{1}{N}\sum B_i A_i \neq (\frac{1}{N}\sum B_i)(\frac{1}{N}\sum A_i)$), aggregating both matrices introduces inherent approximation errors. By maintaining $B_i$ locally, FedSA-LoRA circumvents this algebraic inconsistency, providing a precise and stable foundation for our subsequent scaling factor analysis.

%FedSA-LoRA is extended to other LoRA variants:
%\begin{itemize}
%    \item \textbf{FedSA-rsLoRA:} Incorporates a rank-stabilized scaling factor $\alpha$.
%    \item \textbf{FedSA-VeRA:} Applies to vector-based random matrix adaptation with shared $A_d$ and local $B_b$.
%\end{itemize}

%\subsection*{Formal Model Update}

%Each client's final update after receiving the aggregated matrix is:
%\[
%W_i^{(t+1)} = W_0 + B_i^{(t)} \cdot \bar{A}^{(t)},
%\]
%where $\bar{A}^{(t)}$ is the aggregated matrix at round $t$, and $B_i^{(t)}$ is trained locally.

\section{SFed-LoRA: Stabilized Adapters in Federated Aggregation}

In the SFed-LoRA framework, the scaling factor $\gamma_z$ applied to the matrix product $BA$ is pivotal for maintaining training stability across federated settings. To formalize this requirement, we introduce the concept of a federated-stabilized adapter:

\begin{definition}[\textbf{$(N, r)$-Federated-Stabilized Adapter}]\label{def:federated_stabilized}
Let $z = (N, r)$ denote the tuple comprising client count $N$ and rank $r$. An adapter $\gamma_z BA$ is defined as $(N, r)$-federated-stabilized if it satisfies the following conditions:
\begin{enumerate}
    \item \textbf{Forward Stability:} Given inputs with $h$-th moments scaling as $\Theta_N(1)$ and $\Theta_r(1)$ (per entry), the adapter's output moments must also scale as $\Theta_N(1)$ and $\Theta_r(1)$.
    \item \textbf{Backward Stability:} Given loss gradients with respect to the adapter's outputs scaling as $\Theta_N(1)$ and $\Theta_r(1)$ (per entry), the backpropagated gradients to the inputs must maintain magnitudes of $\Theta_N(1)$ and $\Theta_r(1)$.
\end{enumerate}
\end{definition}

Through asymptotic analysis as $r \to \infty$, we identify that the unique configuration of $\gamma_z$ (up to an additive constant) satisfying Definition \ref{def:federated_stabilized} is:
\begin{equation}
    \gamma_z = \alpha\sqrt{\frac{N}{r}},
\label{eq:gamma_z}
\end{equation}
where $\alpha$ is a hyperparameter. This scaling factor effectively balances the statistical properties of client aggregation ($N$) and rank expansion ($r$), thereby underpinning the stability of the learning process.

\paragraph{Physical Significance.}
The $(N, r)$-federated-stabilized definition ensures that the adapter acts as a \textbf{variance-preserving transformation}, maintaining consistent statistical properties of input data—specifically the $h$-th moment—across variations in client count $N$ and rank $r$. Physically, this stability prevents the signal attenuation (gradient collapse) or amplification (explosion) that is typically exacerbated by the aggregation of $N$ local updates. As illustrated in Figure~\ref{fig2}, suboptimal scaling (e.g., $\gamma = \alpha/r$) leads to \textbf{stagnant convergence} at higher ranks, rendering the additional parameters ineffective. By explicitly compensating for aggregation dynamics, $\gamma_z$ extends the numerical stability—originally observed in standalone rsLoRA—to federated environments. This ensures that high-rank adapters remain trainable and effective even under distributed aggregation, enabling the utilization of their full expressive capacity without incurring inference latency (as adapters are merged into $W_0$ post-training).

Through asymptotic analysis as $r \to \infty$, we derive the unique configuration for $\gamma_z$ that satisfies this stability criterion. The following theorem formalizes this condition:

\begin{theorem}[Optimal Federated Scaling Factor]\label{thm:SFedLoRA_stability}
Consider a federated learning ecosystem with $N$ clients and LoRA adapters scaled by $\gamma_z \in \mathbb{R}$, where $z=(N, r)$. Let the rank $r \to \infty$ and $\gamma_z \to 0$. The adapters are $(N, r)$-federated-stabilized (per Definition \ref{def:federated_stabilized}) if and only if:
\begin{equation}
    \gamma_z = \Theta_z\left(\sqrt{\frac{N}{r}}\right).
\label{eq:gamma_z_theorem}
\end{equation}
In particular, unless $\gamma_z$ scales according to \eqref{eq:gamma_z_theorem}, the learning process exhibits instability or gradient collapse for sufficiently large $r$. This condition holds uniformly throughout the learning trajectory.
\end{theorem}

Theorem \ref{thm:SFedLoRA_stability} establishes that maintaining stability across all ranks in a federated environment mandates scaling adapters with $\gamma_z \propto \sqrt{\frac{N}{r}}$, explicitly reflecting aggregation dynamics across the client size $N$. While the underlying $(N, r)$-federated-stabilized assumptions may appear complex—due to potential dependencies of gradients on $r$ via subsequent adapters—they are simplified by the independence of input data from rank $r$. This facilitates an inductive argument: forward propagation maintains stable output loss gradients, preventing unstable magnitudes as $r \to \infty$, while backpropagation preserves gradient magnitudes within $\Theta_N(1)$ and $\Theta_r(1)$, ensuring stability modulated by the client size $N$.

Guided by this theoretical foundation, the SFed-LoRA framework integrates the rank stabilization approach of \citet{rslora} with the federated aggregation strategy of \citet{fedsa}. The trainable parameters $A_i$ and $B_i$ for each client $i$ evolve based on $\gamma_z$, input $x_n$, and the updated matrices $A_i^{(n)}$ and $B_i^{(n)}$. Appendix A provides a rigorous analysis of this evolution, confirming that $\gamma_z \propto \sqrt{\frac{N}{r}}$ effectively mitigates aggregation-induced instability, aligning with Theorem~\ref{thm:SFedLoRA_stability}.

It is important to note that Theorem~\ref{thm:SFedLoRA_stability} primarily addresses numerical stability and collapse prevention for large $r$, without explicitly modeling feature quality variations across ranks. Since gradient magnitudes may also be influenced by feature quality, deviations from these assumptions could occur if feature quality varies significantly with client size $N$ or rank $r$. Furthermore, if feature quality remains independent, increasing $N$ or $r$ might not justify the associated resource costs. This interplay necessitates the further experimental validation presented in subsequent sections.

\begin{figure*}[t]
    \centering
    \includegraphics[width=0.9\textwidth]{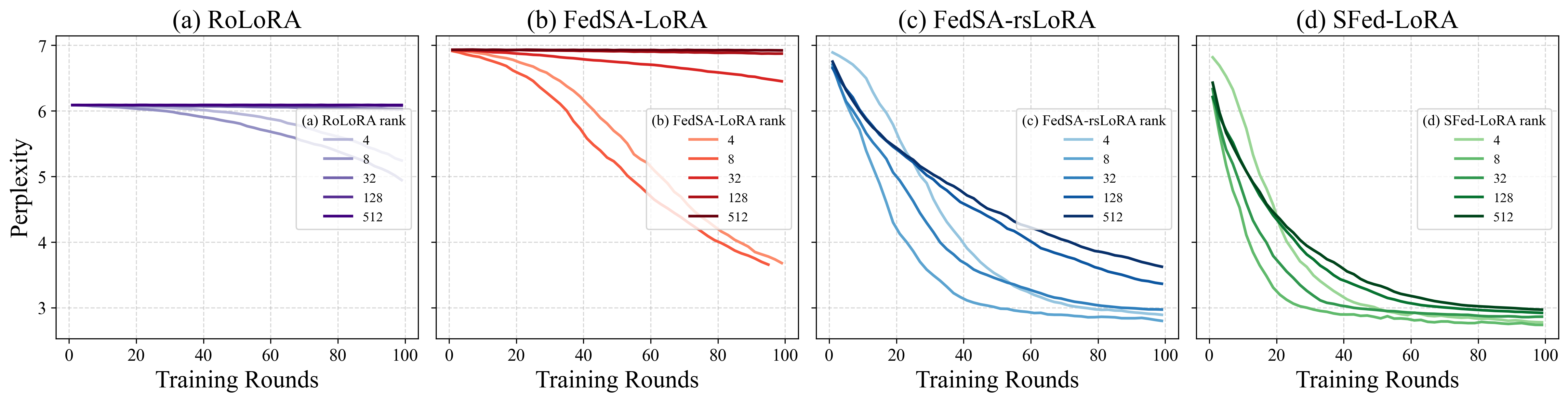}
    \caption{Convergence of Perplexity (PPL) on the Alpaca dataset using the LLaMA2-7B model under an IID federated learning setting. The subplots compare the training trajectories of four methods across ranks \( r \in \{4, 8, 32, 128, 512\} \): (a) RoLoRA (purple), (b) FedSA-LoRA (copper), (c) FedSA-rsLoRA (blue), and (d) SFed-LoRA (green). Darker curves correspond to higher ranks. The figure displays the evolution of validation perplexity over 100 communication rounds.}
    \label{fig2}
\end{figure*}

\begin{figure*}[t]
    \centering
    \includegraphics[width=0.92\textwidth]{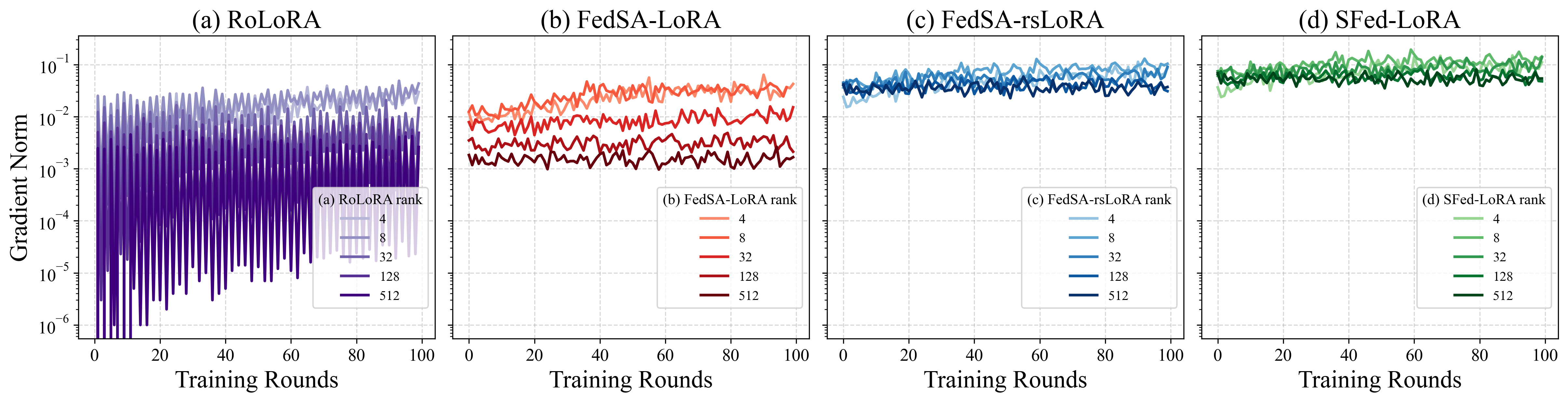}
    \caption{Evolution of average parameter gradient norms on the Alpaca dataset (IID). The subplots display the training trajectories for (a) RoLoRA (purple), (b) FedSA-LoRA (copper), (c) FedSA-rsLoRA (blue), and (d) SFed-LoRA (green) across ranks \( r \in \{4, 8, 32, 128, 512\} \). Darker curves correspond to higher ranks.}
    \label{fig3}
\end{figure*}

\section{Experimental Results}

This section validates the effectiveness of SFed-LoRA in mitigating stability issues arising from client aggregation and rank variations. Leveraging the theoretical scaling factor \(\gamma_z \propto \sqrt{N/r}\), we run evaluations over three dimensions: (1) training stability and performance consistency across ranks; (2) robustness of performance across sizes of client \(N\); and (3) generalization over diverse tasks and data distributed in heterogeneous conditions. In particular, SFed-LoRA is compared against three baselines: FedSA-LoRA (using standard scaling \(\alpha/r\)), FedSA-rsLoRA (using evolved scaling factor \(\alpha/\sqrt{r}\)), and the alternating optimization framework RoLoRA \cite{rolora}. The results show a superior outcome in terms of convergence and stability over high-rank regimes for SFed-LoRA.

\subsection{Stability with a Fixed Client Size}

\paragraph{Configuration.}
We performed experiments with an existing pre-trained LLaMA2-7B-hf model~\citep{llama} on Alpaca dataset~\citep{alpaca}, a standard benchmark for instruction following. To simulate a federated environment, we utilized the FederatedScope-LLM framework~\citep{federatedscope}, where we split the dataset into $3$ clients with identically and independently distributed (IID) configuration following the training protocol in~\citep{fedsa}. The training process spanned 100 communication rounds, with each client performing 10 local update steps per round using the Stochastic Gradient Descent (SGD) optimizer~\citep{SGD}. Consistent with standard LoRA implementations~\citep{lora}, adapter modules were applied to the query (\(\texttt{W}_{\text{q}}\)) and value (\(\texttt{W}_{\text{v}}\)) attention matrices. To systematically evaluate stability across varying model capacities, we swept the rank parameter \( r \in \{4, 8, 32, 128, 512\} \), while fixing the scaling hyperparameter at $\alpha = 8$. The learning rate was set at $0.005$ for fairness across all the baselines (FedSA-LoRA, FedSA-rsLoRA, and RoLoRA). All experiments were executed in half-precision (FP16) on NVIDIA GeForce RTX vGPU-32GB hardware.
 
\paragraph{Results.} Figure~\ref{fig2} shows the perplexity evolution over all communication rounds, with subplots for RoLoRA, FedSA-LoRA, FedSA-rsLoRA, and SFed-LoRA, respectively, where the darker lines point to higher ranks. Consistent with~\citep{parameter}, higher ranks should give a better model performance. However, both FedSA-LoRA (copper) and RoLoRA (purple) have considerable challenges in convergence at higher ranks (e.g., 512), where curves show convergence stagnation due to optimization interference or suboptimal scaling. While FedSA-rsLoRA (blue) mitigates this collapse, it still shows a slow convergence lag at higher ranks. Compared with these baselines, SFed-LoRA’s green curves have the superior stability, achieving the most rapid convergence and consistently lowest perplexity across all ranks. A significant reason behind this is the federated scaling factor \(\gamma_z \propto \sqrt{N/r}\), which could dynamically offset the shifts of variance caused by client aggregation ($N$) and rank expansion ($r$), which thus confirms the robustness of SFed-LoRA in distributed learning contexts.

\paragraph{Gradient Norm Analysis.} Figure~\ref{fig3} presents the average parameter gradient norm across training iterations, with color gradients reflecting rank magnitudes (darker lines denote higher ranks). Observing the baselines, RoLoRA exhibits severe volatility, characterized by high-frequency oscillations where gradient norms periodically plunge to near-zero magnitudes ($< 10^{-6}$), likely due to the disjoint aggregation of alternating updates. FedSA-LoRA displays a distinct stratification where gradient norms collapse exponentially as rank increases; specifically, the gradients at $r=512$ are orders of magnitude lower than those at $r=4$, confirming that the aggressive $\alpha/r$ scaling effectively freezes high-rank parameters. While FedSA-rsLoRA narrows this disparity, it still fails to unify the scales, leaving a noticeable gap between rank extremes. In strong contrast, SFed-LoRA outperforms all baselines by maintaining consistent gradient norms tightly clustered within a stable effective range. Notably, the gradient trajectories for $r=512$ in SFed-LoRA remain strictly comparable to those of $r=4$, demonstrating that the $\gamma_z \propto \sqrt{N/r}$ factor successfully eliminates rank-induced variance and prevents the gradient collapse observed in other methods.

\subsection{Stability Across Client Sizes}

\paragraph{Configuration.}
To explicitly validate the influence of client size $N$ within the theoretical framework of Theorem 4.1, this experiment fixes the rank at a high value ($r = 512$) while systematically varying $N$, complementing the rank-variable setting in Section 5.1. We evaluated stability across varying client counts $N \in \{5, 10, 15, 20\}$, a range simulating the typical scale of cross-silo federated learning scenarios. This setup rigorously tests the algorithm's robustness and the effectiveness of the scaling factor $\gamma_z$ in increasingly distributed environments. All other experimental configurations and hyperparameters remained consistent with those outlined in Section 5.1.

\begin{figure}[ht]
    \centering
    \includegraphics[width=0.9\columnwidth]{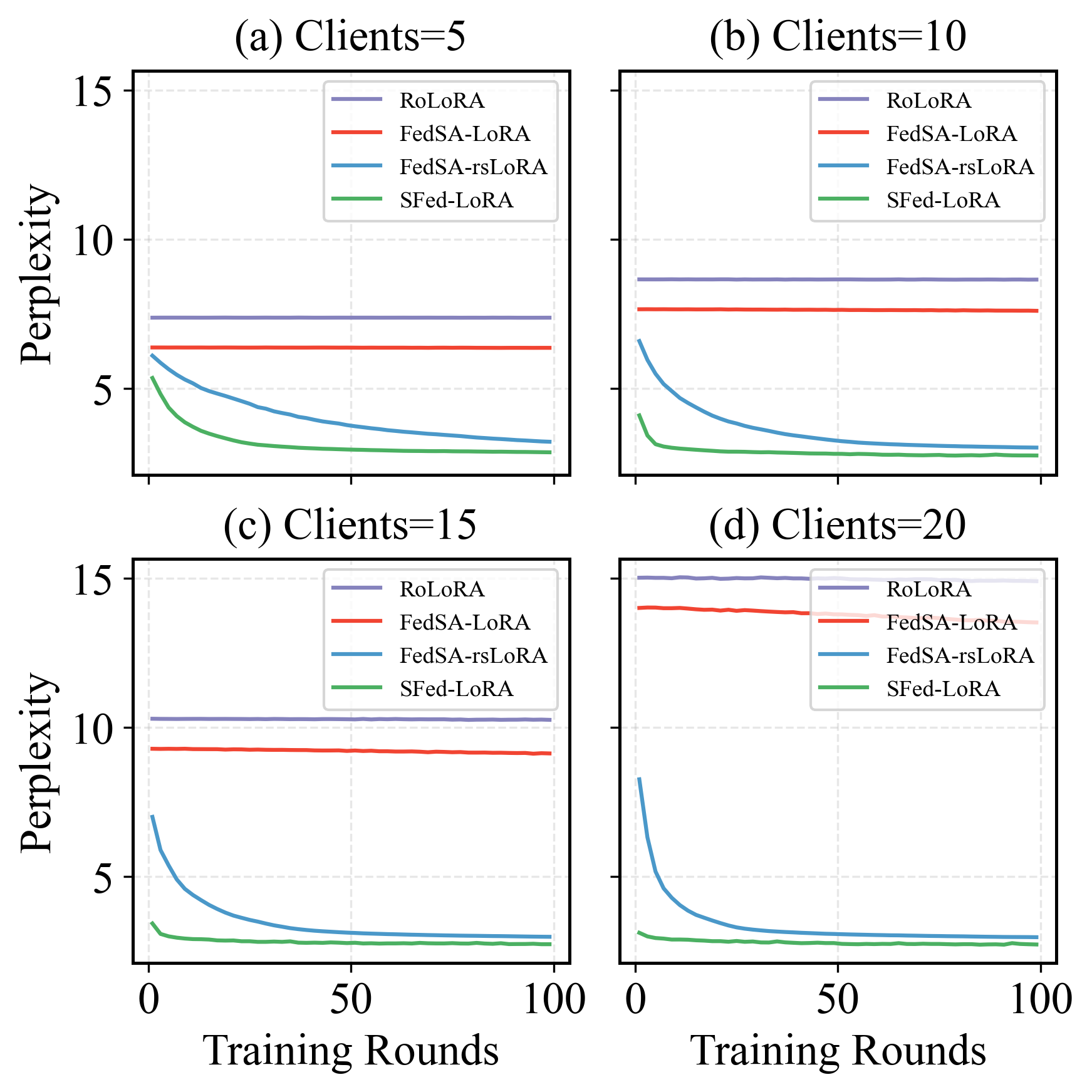}
    \caption{Comparative analysis of perplexity with a fixed rank \( r = 512 \) using the LLaMA2-7B-hf model on the Alpaca dataset in an IID federated learning setting. The subplots correspond to varying client counts: (a) \( N=5 \), (b) \( N=10 \), (c) \( N=15 \), and (d) \( N=20 \). Within each plot, the curves represent RoLoRA (purple), FedSA-LoRA (orange), FedSA-rsLoRA (blue), and SFed-LoRA (green).}
    \label{fig4}
\end{figure}

\paragraph{Results.}
Figure~\ref{fig4} provides a detailed breakdown of convergence performance across varying client sizes (\( N \in \{5, 10, 15, 20\} \)) under a fixed high-rank condition (\( r = 512 \)). The experimental outcomes reveal a critical scalability bottleneck rooted in the choice of scaling factors. Specifically, both RoLoRA (purple) and FedSA-LoRA (orange) exhibit stagnant loss curves with persistently high perplexity. This stagnation stems from their reliance on the original LoRA scaling factor $\gamma = \alpha/r$, which excessively dampens parameter updates at high ranks, effectively suppressing the model's feature learning capacity. Consequently, as the client count increases from $N=5$ to $N=20$, the baseline perplexity deteriorates significantly, jumping from approximately 7 to nearly 15, confirming that aggregating unscaled updates from more clients exacerbates instability. While FedSA-rsLoRA (blue) manages to converge by adopting the stabilized scaling $\gamma_r = \alpha/\sqrt{r}$, it consistently trails behind in performance as it neglects the aggregation dimension. In sharp contrast, SFed-LoRA (green) demonstrates remarkable invariance to client expansion. Regardless of whether \( N=5 \) or \( N=20 \), SFed-LoRA rapidly converges to the lowest perplexity state (approx. 3.0) within the first 20 rounds. This empirical evidence validates that the proposed scaling factor \(\gamma_z \propto \sqrt{N/r}\) effectively neutralizes the variance accumulation inherent in federated aggregation, ensuring consistent high-rank stability independent of the network scale.

\subsection{Generalization Across Models and Datasets}

\paragraph{Configuration.}
To strictly evaluate the generalizability of SFed-LoRA beyond standard instruction following, we expanded our experimental scope to cover diverse tasks, model architectures, optimization algorithms, and data distributions. The evaluation is conducted across two distinct scenarios:

\begin{itemize}
    \item \textbf{Task Diversity (GSM8K):} To verify the method's adaptability across different NLP domains, we extended our evaluation to the GSM8K dataset~\citep{gsm8k}, a standard benchmark for mathematical problem solving. In this setup, we retained the LLaMA2-7B architecture and the configuration from Section 5.1 (SGD optimizer, IID partition) to specifically isolate the impact of task domain variation on method performance.
    
    \item \textbf{NLU \& Heterogeneity (GLUE):} To further test robustness against architectural and distributional shifts, we adopted the GLUE benchmark~\citep{glue} for Natural Language Understanding (NLU). This setup introduces three critical variations: (1) \textbf{Architecture Shift}: Utilizing the RoBERTa-large model~\citep{roberta} (encoder-only) instead of LLaMA (decoder-only); (2) \textbf{Optimizer Shift}: Replacing SGD with the adaptive AdamW optimizer~\citep{adamw} (learning rate \(5 \times 10^{-5}\)) to verify stability under different optimization dynamics; and (3) \textbf{Data Heterogeneity}: Partitioning the dataset across three clients using a non-Identically and Independently Distributed (non-IID) configuration modeled via a Dirichlet distribution with \(\alpha = 0.5\) (\(\text{Dir}(0.5)\)), simulating realistic statistical heterogeneity.
\end{itemize}

Accuracy was selected as the primary evaluation metric for both scenarios, with all other unmentioned hyperparameters consistent with the primary experiments.

\setlength{\tabcolsep}{3pt} 
\begin{table}[h]
    \centering
    \small
    \begin{tabular}{lccc}
        \toprule
        \textbf{Method} & \multicolumn{3}{c}{\textbf{GSM8K Accuracy (\%)}} \\
        \cmidrule(lr){2-4}
        & Rank = 4 & Rank = 8 & Rank = 32 \\
        \midrule
        RoLoRA          & $14.16 \pm 1.18$ & $15.84 \pm 1.18$ & $15.00 \pm 0.00$ \\
        FedSA-LoRA      & $13.33 \pm 0.00$ & $13.89 \pm 0.96$ & $15.56 \pm 0.96$ \\
        FedSA-rsLoRA    & $14.44 \pm 0.96$ & $16.67 \pm 3.34$ & $16.11 \pm 0.96$ \\
        \rowcolor{gray!30} SFed-LoRA & \textbf{14.44} $\pm$ \textbf{1.93} & \textbf{17.08} $\pm$ \textbf{2.08} & \textbf{16.66} $\pm$ \textbf{1.67} \\
        \midrule
        & Rank = 128 & Rank = 512 & \\
        \midrule
        RoLoRA          & $14.16 \pm 1.18$ & $16.66 \pm 2.35$ & \\
        FedSA-LoRA      & $15.56 \pm 0.96$ & $14.44 \pm 0.96$ & \\
        FedSA-rsLoRA    & $15.00 \pm 0.00$ & $15.56 \pm 0.96$ & \\
        \rowcolor{gray!30} SFed-LoRA & \textbf{16.11} $\pm$ \textbf{0.96} & \textbf{17.22} $\pm$ \textbf{1.92} & \\
        \bottomrule
    \end{tabular}
    \caption{Accuracy (\%) comparison on the GSM8K benchmark across ranks. Results are reported as mean $\pm$ standard deviation.}
    \label{accuracygsm8k}
\end{table}

\begin{table}[h]
    \centering
    \small
    \begin{tabular}{lccc}
        \toprule
        \textbf{Method} & \multicolumn{3}{c}{\textbf{MNLI-m Accuracy (\%)}} \\
        \cmidrule(lr){2-4}
        & Rank = 4 & Rank = 8 & Rank = 32 \\
        \midrule
        RoLoRA          & $57.72 \pm 1.51$  & $69.02 \pm 1.37$ & $69.26 \pm 0.89$ \\
        FedSA-LoRA      & $73.89 \pm 10.97$ & $81.31 \pm 5.16$ & $81.88 \pm 3.80$ \\
        FedSA-rsLoRA    & $81.24 \pm 4.08$  & $85.83 \pm 1.39$ & $87.23 \pm 1.65$ \\
        \rowcolor{gray!30} SFed-LoRA & \textbf{83.95} $\pm$ \textbf{2.59} & \textbf{86.51} $\pm$ \textbf{1.34} & \textbf{87.64} $\pm$ \textbf{1.83} \\
        \midrule
        & Rank = 128 & Rank = 512 & \\
        \midrule
        RoLoRA          & $69.74 \pm 0.64$ & $67.84 \pm 2.54$ & \\
        FedSA-LoRA      & $82.22 \pm 3.66$ & $81.25 \pm 5.63$ & \\
        FedSA-rsLoRA    & $87.13 \pm 0.74$ & $87.27 \pm 1.46$ & \\
        \rowcolor{gray!30} SFed-LoRA & \textbf{87.84} $\pm$ \textbf{0.94} & \textbf{87.72} $\pm$ \textbf{2.25} & \\
        \bottomrule
    \end{tabular}
    \caption{Accuracy (\%) comparison on the GLUE benchmark (MNLI-m). Results are reported as mean $\pm$ standard deviation.}
    \label{accuracymnli}
\end{table}
\paragraph{Results.}
\textbf{GSM8K:} Table~\ref{accuracygsm8k} presents the comparative performance on the mathematical reasoning task. While the overall accuracy remains modest due to the intrinsic difficulty of the task, SFed-LoRA demonstrates superior stability compared to the baselines. Specifically, as the rank increases to 512, the standard FedSA-LoRA suffers a performance drop to 14.44, whereas SFed-LoRA maintains a positive trajectory, peaking at 17.22. This represents a margin of 2.78 over the standard baseline and 1.66 over FedSA-rsLoRA. These results confirm that even in diverse task domains, the proposed scaling factor \(\gamma_z\) successfully mitigates the aggregation-induced degradation that hampers high-rank adaptation in other methods.

\textbf{GLUE:} Table~\ref{accuracymnli} details the results on the NLU benchmark. SFed-LoRA consistently outperforms competing methods across the entire rank spectrum, with the advantage becoming most pronounced at higher ranks. At Rank 512, SFed-LoRA achieves an accuracy of 87.72, exceeding the standard FedSA-LoRA (81.25) by a significant margin of 6.47. This substantial gap highlights the failure of the standard $\alpha/r$ scaling to sustain training dynamics under heterogeneous and different optimization conditions. Moreover, SFed-LoRA maintains a consistent lead over the stabilized FedSA-rsLoRA variant across all settings (e.g., +2.71 at Rank 4). 

Collectively, these findings validate that SFed-LoRA provides a robust and generalized solution for federated fine-tuning, independent of model architecture, optimizer choice, or data distribution.

\subsection{Additional Ablation Experiments}

To rigorously validate the robustness of SFed-LoRA under broader conditions, we conducted a series of supplementary ablation studies. Due to space constraints, the comprehensive results and detailed analyses are deferred to Appendix B, with the key findings summarized as follows:

\begin{itemize}
    \item \textbf{Gradient Norm Generalization:} We extended the gradient norm analysis to the GSM8K and GLUE benchmarks (using RoBERTa and AdamW). The results mirror the stability observed in the primary experiments, confirming that SFed-LoRA maintains consistent optimization dynamics across diverse model architectures, optimizers, and datasets.
    
    \item \textbf{Adapter Placement:} We evaluated the impact of applying adapters to all linear modules within the self-attention mechanism (i.e., including \(W_k\) and \(W_o\)), a common variation in LoRA implementations. The performance trends remained consistent with the primary configuration (targeting \(W_q\) and \(W_v\)), demonstrating SFed-LoRA's adaptability to different adapter placement strategies.
    
    \item \textbf{Scaling Factor Verification:} We benchmarked the proposed \(\gamma_z\) against alternative scaling heuristics. Empirical results indicate that deviating from the derived scaling law leads to suboptimal convergence or instability, thereby experimentally validating the optimality of the theoretical formulation \(\gamma_z \propto \sqrt{N/r}\).
\end{itemize}

\section{Conclusion}

In this work, we establish a rigorous theoretical foundation characterizing the critical interplay between the scaling factor, client aggregation size (\(N\)), and adapter rank (\(r\)) in federated fine-tuning. Addressing the gradient collapse observed in high-rank settings, we propose \textbf{SFed-LoRA}, which integrates the statistically derived scaling factor \(\gamma_z =\alpha \sqrt{N/r}\) to neutralize aggregation-induced variance. This formulation ensures consistent output moments and stable gradient norms, effectively decoupling optimization stability from the scale of the federated network. Extensive empirical evaluations across diverse tasks, model architectures, and heterogeneous data distributions confirm that SFed-LoRA significantly outperforms existing baselines by preventing stagnation and enabling rapid convergence. The proposed approach provides a robust solution for distributed LLM adaptation, and future work will focus on extending these stability principles to a broader range of federated large model architectures and complex aggregation protocols.
\bibliography{example_paper}
\bibliographystyle{icml2026}

\newpage
\appendix
\onecolumn
% ----------- Supplementary Content Starts Here -----------
\section{Proof of {\((N, r)\)-Federated-Stability} in SFed-LoRA}

The theoretical analysis of this paper relies on the FedSA-LoRA framework. Below, we derive the conditions for the scaling factor \(\gamma_z\) to ensure \((N, r)\)-federated-stabilized adapters, extending the approach of rsLoRA to the federated setting. The federated learning system has $N$ clients. For client $i$, its LoRA adapters are defined as \(\gamma_z B_i A_i\), where \(B_i \in \mathbb{R}^{d \times r}\), \(A_i \in \mathbb{R}^{r \times k}\) are initialized such that \(B_i^{(0)} = 0^{d \times r}\), and the entries of \(A_i^{(0)}\) are i.i.d.\ with mean 0 and variance \(\sigma_A^2\), independent of both client size \(N\) and rank \(r\). The scaling factor \(\gamma_z \in \mathbb{R}^+\) satisfies \(\gamma_z \to 0\) as \(r \to \infty\). We analyze the stability of adapter outputs and gradients under FedSA-LoRA update rules across \(N\) clients. We define the \((N, r)\)-federated-stabilized adapters as follows:

\begin{definition}\label{def-stability}
The federated learning system has $N$ clients, and each client $i$ has the adapter \(\gamma_z B_i A_i\). Each adapter \(\gamma_z B_i A_i\) is said to be \((N, r)\)-federated-stabilized if:
\begin{enumerate}
    \item If the inputs to the adapter are i.i.d.\ with the \(h\)-th moment scaling as \(\Theta_N(1)\) and \(\Theta_r(1)\) in each entry, then the \(h\)-th moment of the adapter outputs also scales as \(\Theta_N(1)\) and \(\Theta_r(1)\) in each entry.
    \item If the gradient of the loss with respect to the adapter outputs is \(\Theta_N(1)\) and \(\Theta_r(1)\) in each entry, then the loss gradients with respect to the adapter inputs are also \(\Theta_N(1)\) and \(\Theta_r(1)\) in each entry.
\end{enumerate}
A federated learning system is said to be \((N, r)\)-federated-stabilized if and only if all its adapters \(\{\gamma_z B_i A_i\}_{i=1}^{N}\) are \((N, r)\)-federated-stabilized, which ensures the stability of adapter outputs and gradients for large \(r\). 
\end{definition}

We aim to prove the following:

\begin{theorem}\label{theorem}
In a federated learning system, all adapters \(\{\gamma_z B_i A_i\}_{i=1}^{N}\) are \((N, r)\)-federated-stabilized
%For FedSA-LoRA adapters \(\gamma_z B_i A_i\), where \(B_i^{(n)}\) and \(A_i^{(n)}\) evolve according to the update rules, the adapters are \((N, r)\)-stabilized 
if and only if \(\gamma_z \in \Theta_z(\sqrt{\frac{N}{r}})\), where \(z = (N, r)\), $N$ denotes the client size, and \(r\) denotes the rank.
\end{theorem}

\subsection{Notation}

Before the proof of Theorem \ref{theorem}, we explain some notation used in Definition \ref{def-stability}. While computing the $n$-th SGD update for client $i$ with a learning rate \( \eta \), we have:
\begin{itemize}
    \item \( x_{i,n} \) is the input vector, and \( \gamma_z \) is the scaling factor; 
    \item Let \( B_i^{(n)} \in \mathbb{R}^{d \times r} \) and \( A_i^{(n)} \in \mathbb{R}^{r \times k} \) denote the adapter matrices after the $n$-th SGD update on input \( x_{i,n} \) for client \( i \), where \( B_i^{(n)} \) is initialized as \( B_i^{(0)} = 0^{d \times r} \), and the entries of \( A_i^{(0)} \) are i.i.d.\ with mean 0 and variance \( \sigma_A^2 \), independent of both client size \( N \) and rank \( r \);
    \item Let \( f_i(x_{i,n}) = \gamma_z B_i^{(n)} A_i^{(n)} x_{i,n} \) denote the adapter output for client \( i \);
    \item Let \( \mathcal{L}(f_i(x_{i,n})) \) represent the corresponding loss function. The gradient of the loss with respect to the adapter output is denoted by \( v_{i,n} = \nabla_{f_i(x_{i,n})} \mathcal{L}(f_i(x_{i,n})) \), reflecting the sensitivity of the loss to the output, while \( \nabla_{x_{i,n}} \mathcal{L} \) denotes the loss gradient with respect to the input \( x_{i,n} \);
    \item The $n$-th SGD update corresponds to the $t$-th round of communication in the federated learning training process, where $t = n$.
\end{itemize}

The symbol \(\Theta\) defines exact asymptotic behavior \citep{yang2020feature}. Given a sequence of scalar random variables \( c = \{c_l \in \mathbb{R}\}_{l=1}^{\infty} \), we write \( c = \Theta(l^{-a}) \) if there exist constants \( C_1, C_2 \) such that \( C_1 l^{-a} \leq |c| \leq C_2 l^{-a} \) for sufficiently large \( l \), almost surely. Given a sequence of random vectors \( x = \{x_l \in \mathbb{R}^l\}_{l=1}^{\infty} \), we say \( x \) has coordinates of size \( \Theta(l^{-a}) \) and write \( x = \Theta(l^{-a}) \) to mean the scalar random variable sequence \( \{\sqrt{\|x_l\|^2/l}\}_l \) is \( \Theta(l^{-a}) \). Similarly for the notations \( O(l^{-a}) \), \( \Omega(l^{-a}) \). We use the notations \( \Theta_\xi(l^{-a}) \), \( O_\xi(l^{-a}) \), \( \Omega_\xi(l^{-a}) \) if the hidden constants \( C_1, C_2 \) are allowed to depend on some object \( \xi \). 

\subsection{Proof of Theorem \ref{theorem}}

\begin{proof}
As shown in \citep{fedsa}, for all clients \(i\), we have: 
\begin{align}
\nabla_{B_i^{(n)}} \mathcal{L} &= \gamma_z v_{i,n} x_{i,n}^\top (A_i^{(n)})^\top, \label{eq:grad_B_i} \\
\nabla_{A_i^{(n)}} \mathcal{L} &= \gamma_z (B_i^{(n)})^\top v_{i,n} x_{i,n}^\top. \label{eq:grad_A_i}
\end{align}
Let \(\bar{A} = \frac{1}{N} \sum_{j=1}^N A_j^{(0)}\) denote the average initial matrix \(A\) across \(N\) clients. Below, we consider the federated learning process and derive the expressions for \(B_i^{(n)}\) and \(A_i^{(n)}\) of client $i$ when \(n = 0, 1, 2\).

\vspace{0.33em}\noindent\textbf{Derivation while computing the initial SGD update with \(n = 0\):}

\noindent-- Initially, we have \(B_i^{(0)} = 0^{d \times r}\), and \( A_i^{(0)} \sim \mathcal{N}(0, \sigma_A^2) \in \mathbb{R}^{r \times k} \).

\noindent-- Gradient:
  \begin{equation} 
  \nabla_{B_i^{(0)}} \mathcal{L} = \gamma_z v_{i,0} x_{i,0}^\top (A_i^{(0)})^\top,
  \end{equation}
  \begin{equation} 
  \nabla_{A_i^{(0)}} \mathcal{L} = \gamma_z (B_i^{(0)})^\top v_{i,0} x_{i,0}^\top = 0.
  \end{equation}
\noindent-- Update at client $i$:
  \begin{equation} 
  B_i^{(1)} = B_i^{(0)} - \eta \nabla_{B_i^{(0)}} \mathcal{L} = -\eta \gamma_z v_{i,0} x_{i,0}^\top (A_i^{(0)})^\top,
  \end{equation}
  \begin{equation} 
  A_i^{\text{pre}(1)} = A_i^{(0)} - \eta \nabla_{A_i^{(0)}} \mathcal{L} = A_i^{(0)}.
  \end{equation}
-- After aggregation on the central server and download to the client $i$, \(A_i^{(1)} = \bar{A}\).

\vspace{0.33em}\noindent\textbf{Derivation with \(n = 1\):}

\noindent-- Gradient:
  \begin{equation}  
  \nabla_{B_i^{(1)}} \mathcal{L} = \gamma_z v_{i,1} x_{i,1}^\top (A_i^{(1)})^\top = \gamma_z v_{i,1} x_{i,1}^\top (\bar{A})^\top,
  \end{equation}
  \begin{equation} 
  \nabla_{A_i^{(1)}} \mathcal{L} = \gamma_z (B_i^{(1)})^\top v_{i,1} x_{i,1}^\top.
  \end{equation}
-- Update at client $i$:
  \begin{equation}  
  B_i^{(2)} = B_i^{(1)} - \eta \nabla_{B_i^{(1)}} \mathcal{L} = -\eta \gamma_z v_{i,0} x_{i,0}^\top (A_i^{(0)})^\top - \eta \gamma_z v_{i,1} x_{i,1}^\top (\bar{A})^\top,
  \end{equation} 
  \begin{equation}  
  \begin{split}  
  A_i^{\text{pre}(2)} & = A_i^{(1)} - \eta \nabla_{A_i^{(1)}} \mathcal{L} = \bar{A} - \eta \gamma_z (B_i^{(1)})^\top v_{i,1} x_{i,1}^\top \\
  & = \bar{A} - \eta \gamma_z \left( -\eta \gamma_z v_{i,0} x_{i,0}^\top (A_i^{(0)})^\top \right)^\top v_{i,1} x_{i,1}^\top\\ 
  & = \bar{A} + A_i^{(0)} \left( \eta^2 \gamma_z^2 (x_{i,0} (v_{i,0}^\top v_{i,1}) x_{i,1}^\top) \right).\nonumber
  \end{split}    
  \end{equation}  
-- Like \citep{rslora}, the term \(\eta^2 \gamma_z^2 (x_{i,0} (v_{i,0}^\top v_{i,1}) x_{i,1}^\top)\) is of order \(\mathcal{O}_r(\gamma_z^2)\); 
%as \(\eta\), \(x_{i,0}\), \(v_{i,0}\), and \(v_{i,1}\) are \(\Theta_r(1)\) while \(\gamma_z \to 0\) as \(r \to \infty\), allowing simplification to focus on the dominant term \(\bar{A}\); 
thus, we have:
  \begin{equation}
      A_i^{\text{pre}(2)} = \bar{A} + A_i^{(0)} \mathcal{O}_r(\gamma_z^2).
  \end{equation}
-- After aggregation on the central server and download to the client $i$, \(A_i^{(2)} = \bar{A} \left(1 + \mathcal{O}_r(\gamma_z^2)\right)\).

\vspace{0.33em}\noindent\textbf{Derivation with \(n = 2\):}

\noindent-- Gradient:
  \begin{equation} 
  \nabla_{B_i^{(2)}} \mathcal{L} = \gamma_z v_{i,2} x_{i,2}^\top (A_i^{(2)})^\top = \gamma_z v_{i,2} x_{i,2}^\top (\bar{A} \left(1 + \mathcal{O}_r(\gamma_z^2)\right))^\top,
  \end{equation}
  \begin{equation} 
  \nabla_{A_i^{(2)}} \mathcal{L} = \gamma_z (B_i^{(2)})^\top v_{i,2} x_{i,2}^\top.
  \end{equation}
-- Update at client $i$:
  \begin{equation}
  B_i^{(3)} = B_i^{(2)} - \eta \nabla_{B_i^{(2)}} \mathcal{L} = -\eta \gamma_z \left[ v_{i,0} x_{i,0}^\top (A_i^{(0)})^\top + \sum_{s=1}^{2} v_{i,s} x_{i,s}^\top \left(1 + \mathcal{O}_r(\gamma_z^2)\right) \bar{A}^\top \right],
  \end{equation}  

  \begin{equation}  
  \begin{split}
  A_i^{\text{pre}(3)} & = A_i^{(2)} - \eta \nabla_{A_i^{(2)}} \mathcal{L} = \bar{A} \left(1 + \mathcal{O}_r(\gamma_z^2)\right) - \eta \gamma_z (B_i^{(2)})^\top v_{i,2} x_{i,2}^\top\\
  & = \bar{A} \left(1 + \mathcal{O}_r(\gamma_z^2)\right) - \eta \gamma_z \left( -\eta \gamma_z v_{i,0} x_{i,0}^\top (A_i^{(0)})^\top - \eta \gamma_z v_{i,1} x_{i,1}^\top (\bar{A})^\top \right)^\top v_{i,2} x_{i,2}^\top\\
 & = \bar{A} \left(1 + \mathcal{O}_r(\gamma_z^2)\right) + A_i^{(0)} \left( \eta^2 \gamma_z^2 (x_{i,0} (v_{i,0}^\top v_{i,2}) x_{i,2}^\top) \right) + \bar{A} \left( \eta^2 \gamma_z^2 (x_{i,1} (v_{i,1}^\top v_{i,2}) x_{i,2}^\top) \right).\nonumber
  \end{split}
  \end{equation}
-- Like \citep{rslora}, the terms \(\eta^2 \gamma_z^2 (x_{i,0} (v_{i,0}^\top v_{i,2}) x_{i,2}^\top)\) and \(\eta^2 \gamma_z^2 (x_{i,1} (v_{i,1}^\top v_{i,2}) x_{i,2}^\top)\) are of order \(\mathcal{O}_r(\gamma_z^2)\); thus, we have:
  \begin{equation}
      A_i^{\text{pre}(3)} = \bar{A} \left(1 + \mathcal{O}_r(\gamma_z^2)\right) + A_i^{(0)} \mathcal{O}_r(\gamma_z^2) + \bar{A} \mathcal{O}_r(\gamma_z^2).
  \end{equation}
-- After aggregation on the central server and download to the client $i$, \(A_i^{(3)} = \bar{A} \left(1 + \mathcal{O}_r(\gamma_z^2)\right)\).

%We can verify by mathematical induction that, for a general \(n\), the following holds: Let \(\bar{A} = \frac{1}{N} \sum_{j=1}^N A_j^{(0)}\) denote the average initial \(A\) matrix across \(N\) clients. 

We can verify by mathematical induction that, for \(n \geq 2\), the adapter matrices after the \(n\)-th update for client \(i\) are given by:
\begin{equation}
    B_i^{(n)} = -\eta \gamma_z \left( v_{i,0} x_{i,0}^\top (A_i^{(0)})^\top + \sum_{s=1}^{n-1} v_{i,s} x_{i,s}^\top \left(1 + \mathcal{O}_r(\gamma_z^2)\right) \bar{A}^\top \right),
\end{equation}
\begin{equation}
    A_i^{(n)} = \bar{A} \left(1 + \mathcal{O}_r(\gamma_z^2)\right).
\end{equation}

The adapter output for client \(i\) is:
\begin{equation}
    \gamma_z B_i^{(n)} A_i^{(n)} = -\gamma_z^2 \eta \left( v_{i,0} x_{i,0}^\top (A_i^{(0)})^\top \bar{A} + \sum_{s=1}^{n-1} v_{i,s} x_{i,s}^\top \bar{A}^\top \bar{A} \right) + \mathcal{O}_r(\gamma_z^3) \bar{A}^\top \bar{A}.
\end{equation}

Taking the expectation over the initialization \(A_j^{(0)}\) for all clients \(j\):
\[
\mathbb{E}_{A_j^{(0)}} \left( (A_i^{(0)})^\top \bar{A} \right) = \mathbb{E}_{A_j^{(0)}} \left( \bar{A}^\top \bar{A} \right) = \frac{r}{N} \sigma_A^2 I_{d \times d},
\]
it follows that
\begin{equation}
    \mathbb{E}_{A_j^{(0)}} \left( \gamma_z B_i^{(n)} A_i^{(n)} \right) = -\gamma_z^2 \frac{r}{N} \sigma_A^2 \eta \sum_{s=0}^{n-1} v_{i,s} x_{i,s}^\top + \mathcal{O}_r \left( \gamma_z^3 \frac{r}{N} \right).
\end{equation}

\textbf{(Backward pass:)} On a new input \(x_{i,n}\) for client \(i\), the gradient is
\begin{equation}
    \nabla_{x_{i,n}} \mathcal{L}(\gamma_z B_i^{(n)} A_i^{(n)} x_{i,n}) = -\gamma_z^2 \frac{r}{N} \sigma_A^2 \eta \sum_{s=0}^{n-1} x_{i,s} v_{i,s}^\top v_{i,n} + \mathcal{O}_r(\gamma_z^3 \frac{r}{N}) \in \Theta_z(\gamma_z^2 \frac{r}{N}).
\end{equation}

\textbf{(Forward pass:)} All inputs are i.i.d.\ with their moments satisfying \(\Theta_N(1)\) and \(\Theta_r(1)\). Thus, on a new input \(x_{i,n}\) for client \(i\), we have \(\mathbb{E}_x((x_{i,s}^\top x_{i,n})^h) \in \Theta_N(1)\) and \(\Theta_r(1)\). With \(\gamma_z \to 0\), the expectation yields
\begin{equation}
    \mathbb{E}_{x, A_j^{(0)}} \left( (\gamma_z B_i^{(n)} A_i^{(n)} x_{i,n})^h \right) = \left( -\gamma_z^2 \frac{r}{N} \sigma_A^2 \eta \right)^h \sum_{s=0}^{n-1} v_{i,s}^h \mathbb{E}_x((x_{i,s}^\top x_{i,n})^h) + \mathcal{O}_r \left( (\gamma_z^3 \frac{r}{N})^h \right) \in \Theta_z((\gamma_z^2 \frac{r}{N})^h).
\end{equation}

To prevent instability or collapse as \( r \to \infty \), the output moments and gradient norms must remain \(\Theta_N(1)\) and \(\Theta_r(1)\), which is satisfied if and only if \(\Theta_z((\gamma_z^2 \frac{r}{N})^h) = \Theta_N(1) \text{ and } \Theta_r(1)\), or equivalently
\[
\gamma_z \in \Theta_z(\sqrt{\frac{N}{r}}). 
\]
\end{proof}

\section{Ablations and Additional Experiments}

\subsection{Gradient Norm Generalization}
\begin{figure}[htb]
    \centering
    \includegraphics[width=1.0\textwidth]{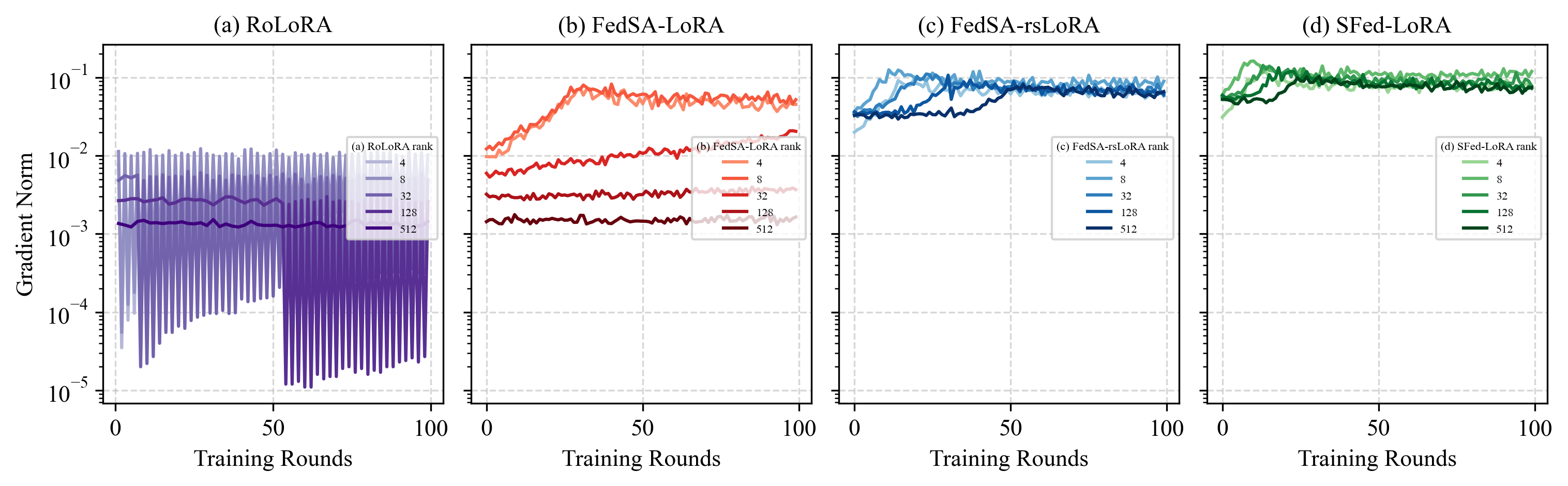}
    \caption{Evolution of average parameter gradient norms on the GSM8K dataset using the LLaMA2-7B model. The subplots display the training trajectories for (a) RoLoRA, (b) FedSA-LoRA, (c) FedSA-rsLoRA, and (d) SFed-LoRA across ranks \( r \in \{4, 8, 32, 128, 512\} \). Darker curves correspond to higher ranks.}
    \label{fig:gsm8k_grad}
\end{figure}

\begin{figure}[htb]
    \centering
    \includegraphics[width=1.0\textwidth]{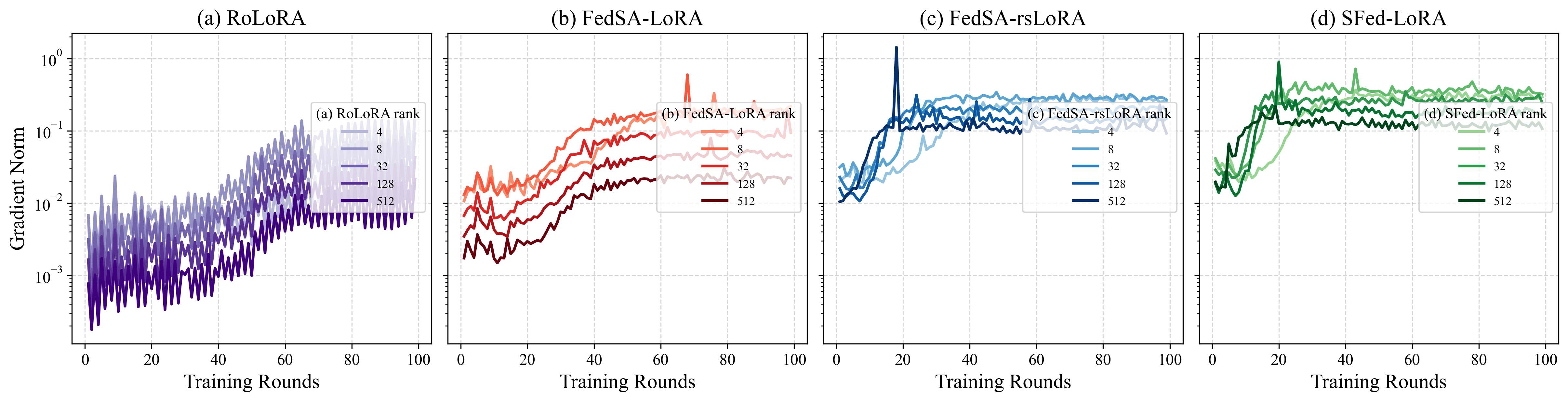}
    \caption{Evolution of average parameter gradient norms on the GLUE benchmark using the RoBERTa-large model with the AdamW optimizer. The subplots compare (a) RoLoRA, (b) FedSA-LoRA, (c) FedSA-rsLoRA, and (d) SFed-LoRA across ranks \( r \in \{4, 8, 32, 128, 512\} \). Darker curves correspond to higher ranks.}
    \label{fig:glue_grad}
\end{figure}
\paragraph{Configuration.}
To strictly evaluate the generalizability of SFed-LoRA beyond the primary setting, we extended the gradient norm analysis to cover diverse task domains and optimization landscapes. This section details the training dynamics for two distinct configurations: (1) \textbf{Task Diversity:} Using LLaMA2-7B on the GSM8K dataset with SGD, isolating the impact of the reasoning task; and (2) \textbf{Optimization Shift:} Using RoBERTa-large on the GLUE benchmark with the AdamW optimizer, testing robustness against adaptive moment estimation and encoder-only architectures. The rank parameters were swept across \( r \in \{4, 8, 32, 128, 512\} \), consistent with the main experiments.

\paragraph{Results on GSM8K (LLaMA2 + SGD).}
Figure~\ref{fig:gsm8k_grad} visualizes the gradient evolution during the fine-tuning of LLaMA2 on the mathematical reasoning task. The baseline behaviors highlight significant instability: RoLoRA (Plot a) exhibits severe high-frequency oscillations with gradient magnitudes periodically collapsing to negligible levels ($<10^{-4}$). FedSA-LoRA (Plot b) displays a clear rank-dependent stratification, where higher ranks (darker red lines) suffer from significantly attenuated gradient norms compared to lower ranks, indicating a suppression of feature learning capacity at large $r$. While FedSA-rsLoRA (Plot c) alleviates this disparity to some extent, distinct gaps between ranks persist. In contrast, SFed-LoRA (Plot d) demonstrates superior stability; the gradient trajectories for all ranks are tightly clustered within the same magnitude range ($10^{-1}$). This invariance confirms that the proposed scaling factor \(\gamma_z\) effectively neutralizes aggregation-induced variance, ensuring consistent optimization dynamics for LLaMA2 even in complex reasoning tasks.

\begin{figure}[htb]
    \centering
    \includegraphics[width=0.8\textwidth]{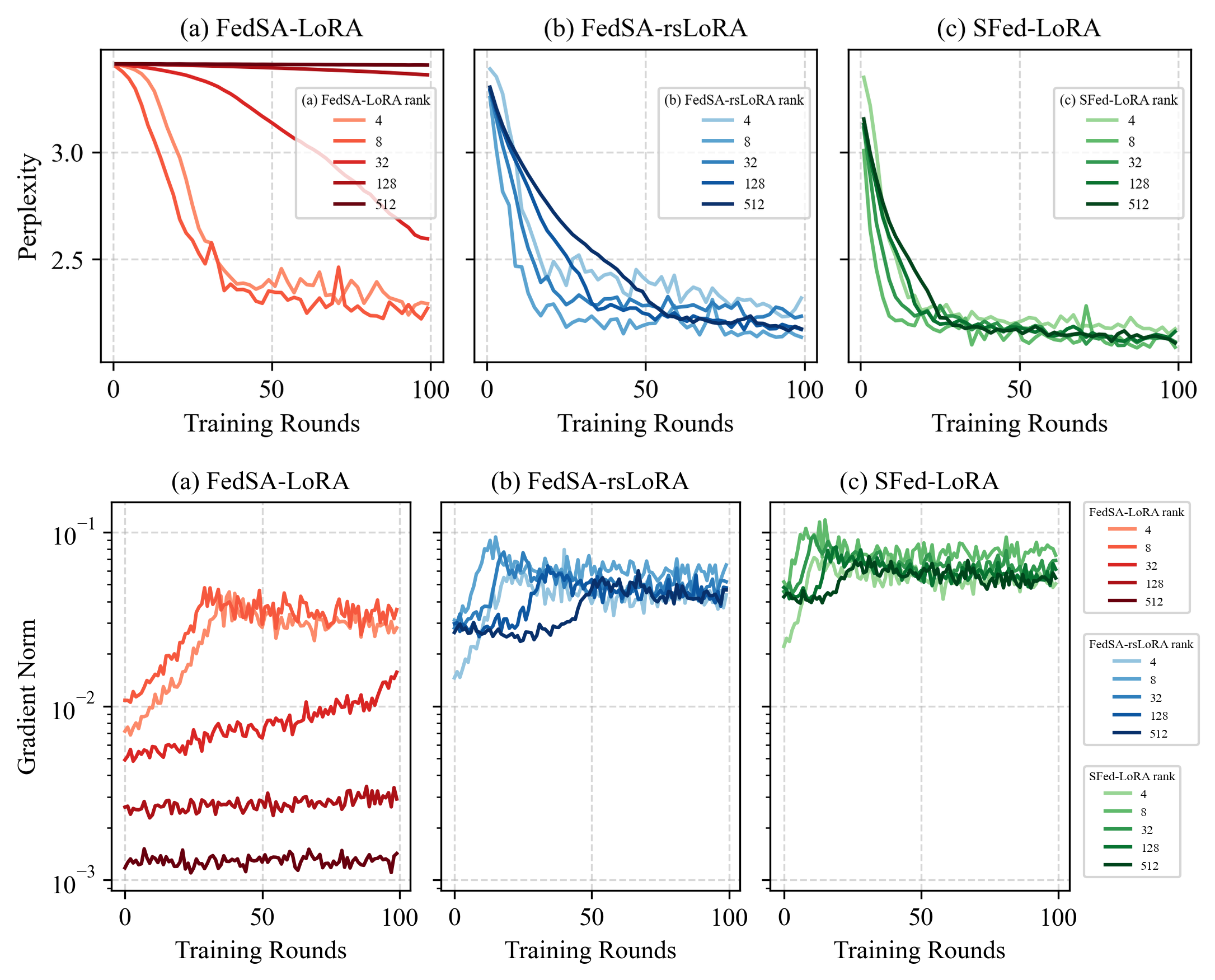}
    \caption{Perplexity (top row) and gradient norms (bottom row) of LLaMA-2-7B fine-tuned on the GSM8K dataset. Adapters are applied to all attention modules (\(\texttt{W}_{\text{q}}, \texttt{W}_{\text{k}}, \texttt{W}_{\text{v}}, \texttt{W}_{\text{o}}\)) across ranks \( r \in \{4, 8, 32, 128, 512\} \). The curves represent FedSA-LoRA (copper), FedSA-rsLoRA (blue), and SFed-LoRA (green), with darker shades indicating higher ranks.}
    \label{app-fig2}
\end{figure}

\paragraph{Results on GLUE (RoBERTa + AdamW).}
Figure~\ref{fig:glue_grad} extends the analysis to a BERT architecture (RoBERTa-large) optimized with AdamW, a setup characterized by adaptive learning rates that typically complicate gradient stability. Despite the change in optimizer, standard FedSA-LoRA (Plot b) continues to exhibit gradient collapse at higher ranks, evidenced by the distinct separation between the light (low rank) and dark (high rank) curves. RoLoRA (Plot a) remains volatile, struggling to maintain a stable optimization trajectory. SFed-LoRA (Plot d), however, maintains its robustness; the gradient norms across all ranks ($r=4$ to $r=512$) overlap significantly, showing minimal deviation. This result is particularly crucial as it verifies that the theoretical stability provided by \(\gamma_z = \alpha\sqrt{N/r}\) holds true regardless of the underlying optimizer (SGD vs. AdamW) or model architecture (Decoder-only vs. Encoder-only), confirming SFed-LoRA's broad generalizability in federated settings.

\subsection{Adapters in Attention Modules}

\paragraph{Configuration.}
In the primary experiments, the injection of LoRA modules was restricted to the \(\texttt{W}_{\text{q}}\) and \(\texttt{W}_{\text{v}}\) layers of the attention mechanism, following the standard configuration outlined by \citet{lora}. To investigate the adaptability of the proposed approach under increased parameter density, additional experiments integrated adapters across all linear modules within the self-attention mechanism, including \(\texttt{W}_{\text{q}}\), \(\texttt{W}_{\text{k}}\), \(\texttt{W}_{\text{v}}\), and \(\texttt{W}_{\text{o}}\). All other experimental settings, including the use of the LLaMA-2-7B model and GSM8K dataset, remained consistent with the primary experiments.

\paragraph{Results.}
Figure~\ref{app-fig2} illustrates the training dynamics when adapters are extended to all attention modules. The results align closely with the findings from the primary experiments, establishing a clear link between gradient stability and model convergence.
For FedSA-LoRA (copper plots), a severe gradient collapse is observed in the bottom row; as the rank increases to 512, gradient norms diminish significantly. This numerical instability directly precipitates the stagnation seen in the top row, where high-rank models fail to reduce perplexity.
FedSA-rsLoRA (blue plots) partially mitigates this issue, yet a visible stratification remains—high-rank gradients lag behind lower ones, resulting in slower convergence rates.
In contrast, SFed-LoRA (green plots) demonstrates remarkable robustness. The gradient norms (bottom row) remain tightly clustered and stable across the entire rank spectrum. Consequently, the perplexity curves (top row) show that high-rank adapters converge rapidly and effectively, often outperforming their lower-rank counterparts. These findings confirm that the scaling factor \(\gamma_z = \alpha\sqrt{N/r}\) successfully preserves optimization stability even when the adapter architecture is expanded.

\subsection{Evaluation of Alternative Scaling Factors}

\paragraph{Configuration.}
To rigorously verify the optimality of the proposed scaling factor \(\gamma_z \propto \sqrt{N/r}\), we conducted controlled ablation studies under an extreme high-rank setting (\(r = 2048\)). These experiments were designed to position \(\gamma_z\) against alternative scaling candidates with magnitudes distinctly smaller and larger than our theoretical derivation, thereby validating whether \(\gamma_z\) constitutes the optimal choice. Specifically, to assess the performance of a scaling factor smaller than the proposed \(\gamma_z\), we defined \(\gamma_{za}\) as:
\begin{equation}
\gamma_{za} = \frac{1}{\sqrt{N} \cdot \sqrt{r}} \label{eq:gamma1}
\end{equation}
Conversely, to evaluate the performance of a scaling factor larger than \(\gamma_z\), we introduced a setting defined as:
\begin{equation}
\gamma_{zb} = \frac{N^2}{\sqrt{r}} \label{eq:gamma2}
\end{equation}
All experiments were conducted using the LLaMA-2-7B model on the GSM8K dataset.

\begin{figure}[htb]
    \centering
    \includegraphics[width=0.55\textwidth]{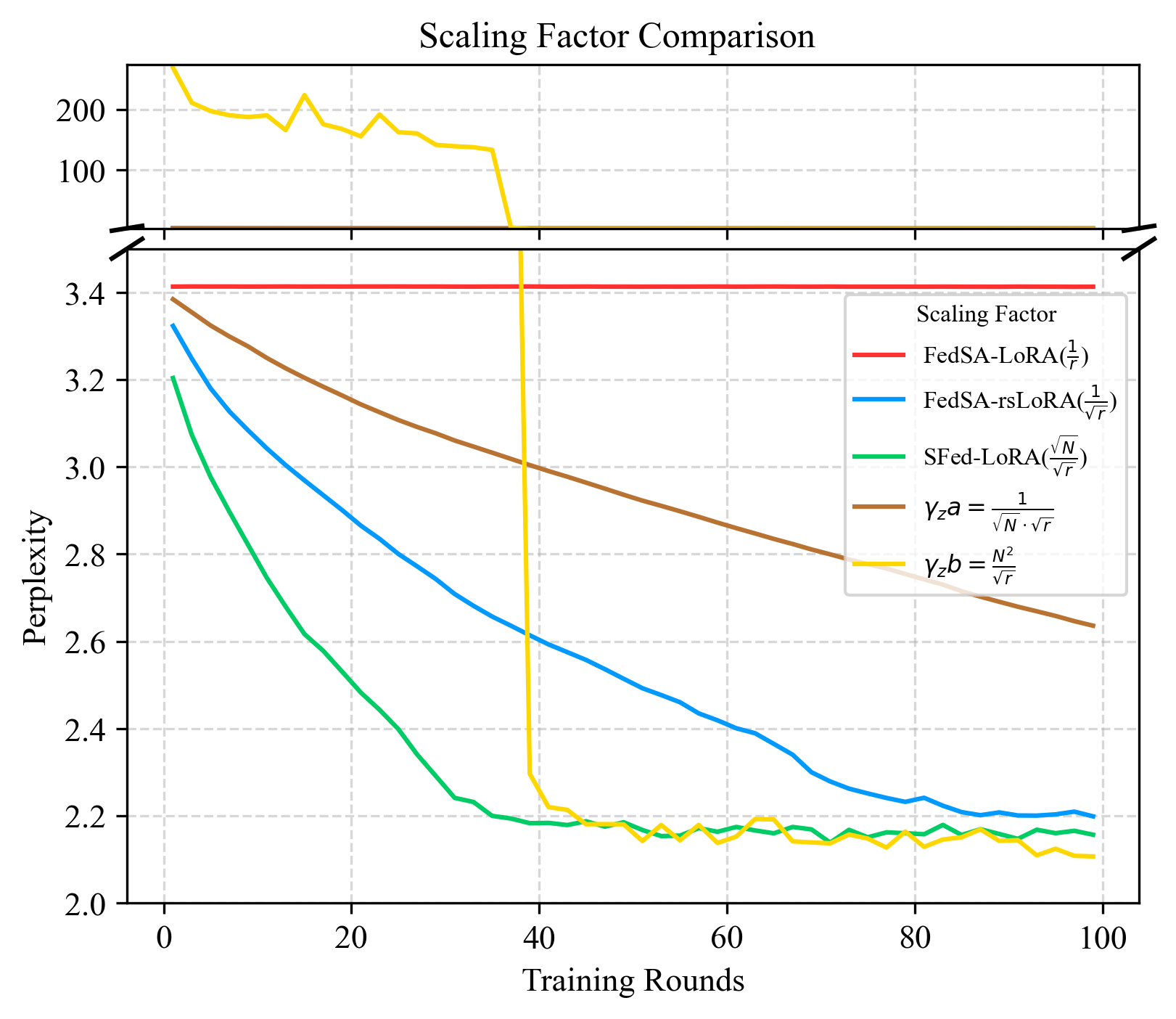}
    \caption{Perplexity trajectories of LLaMA-2-7B on the GSM8K dataset with a fixed high rank \( r = 2048 \). The plot compares the proposed SFed-LoRA (green) against FedSA-LoRA (red), FedSA-rsLoRA (blue), and alternative heuristics \(\gamma_{za}\) (brown) and \(\gamma_{zb}\) (yellow). Note the broken y-axis used to visualize the extreme initial perplexity of \(\gamma_{zb}\).}
    \label{app-fig3}
\end{figure}

\paragraph{Results.}
Figure~\ref{app-fig3} illustrates the convergence behaviors of the different scaling configurations, empirically confirming that \(\gamma_z\) represents the optimal scaling factor for federated fine-tuning. Among all evaluated candidates, the proposed SFed-LoRA (green curve) exhibits the most superior dynamics, achieving the most rapid convergence and the lowest final perplexity ($\approx 2.1$) with robust stability. In contrast, candidates deviating from this theoretical optimum yield suboptimal performance. The larger scaling factor \(\gamma_{zb}\) (yellow curve) results in severe initial volatility with perplexity exploding beyond 200, indicating that a factor exceeding \(\gamma_z\) compromises training stability. Conversely, the smaller scaling factor \(\gamma_{za}\) (brown curve), along with the rank-stabilized baseline (FedSA-rsLoRA, blue curve), exhibits a significantly slower convergence rate due to insufficient scaling magnitude. Meanwhile, the standard \(\alpha/r\) scaling (FedSA-LoRA, red curve) leads to stagnation. These findings conclusively demonstrate that the theoretically derived \(\gamma_z = \alpha\sqrt{N/r}\) effectively strikes the best balance, outperforming both larger and smaller alternatives to ensure the highest training efficiency and stability.

\begin{figure}[H]
    \centering
    \includegraphics[width=0.8\textwidth]{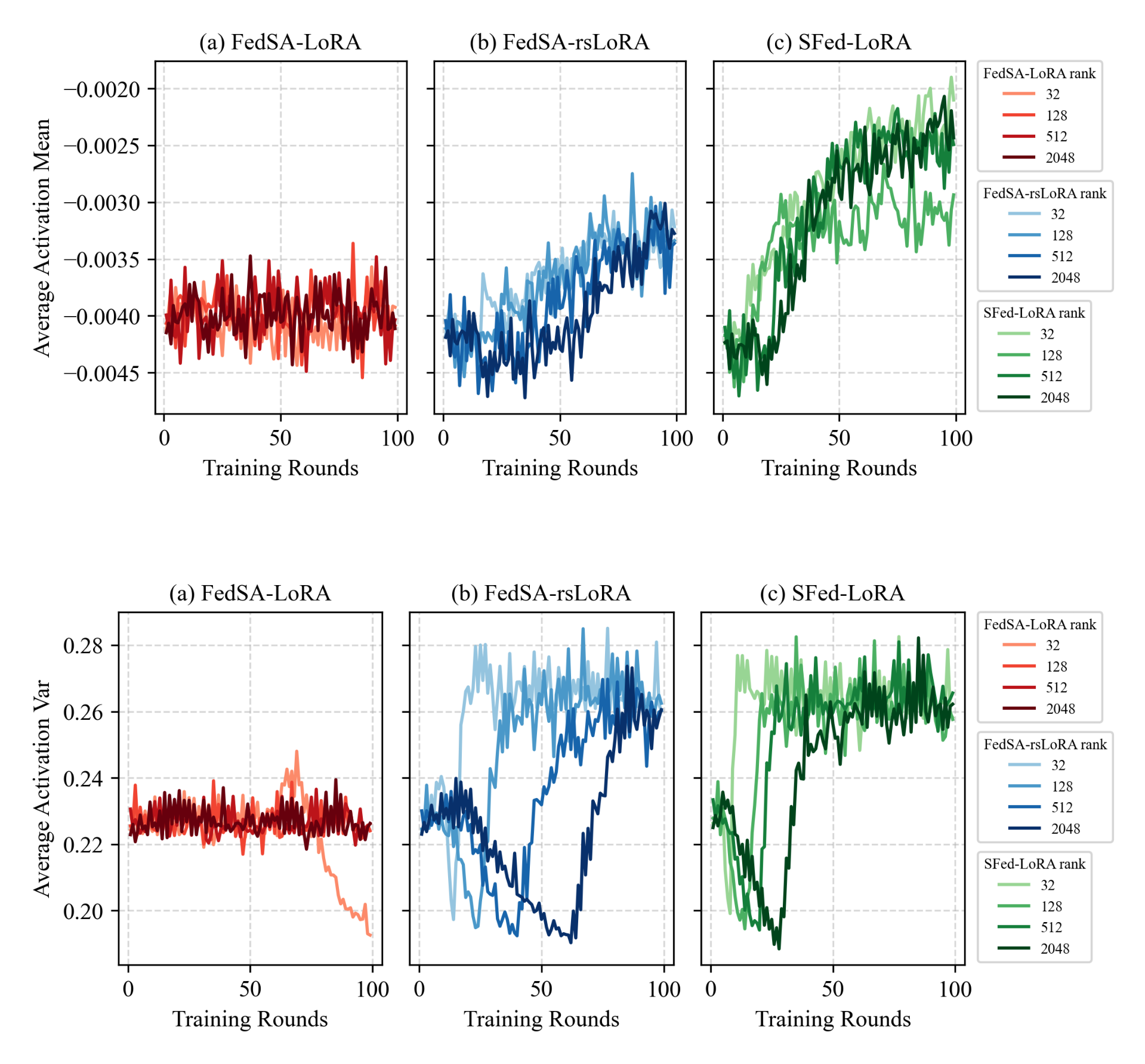}
    \caption{Mean (top row) and variance (bottom row) of post-adapter pre-LayerNorm activations for LLaMA-2-7B on the GSM8K dataset. The plots compare FedSA-LoRA (red), FedSA-rsLoRA (blue), and SFed-LoRA (green) across ranks \( r \in \{32, 128, 512, 2048\} \), with darker shades indicating higher ranks.}
    \label{app-fig4}
\end{figure}
\subsection{Activations}

\paragraph{Configuration.}
Guided by the theoretical implications regarding activation moments discussed in Theorem 4.1, we conducted a comprehensive evaluation of the averaged first two moments (mean and variance) of post-adapter, pre-LayerNorm activations. We compared the activation characteristics of FedSA-LoRA, FedSA-rsLoRA, and SFed-LoRA using the LLaMA-2-7B model on the GSM8K dataset. The experiments covered a broad spectrum of ranks \( r \in \{32, 128, 512, 2048\} \) to assess stability within federated learning settings.

\paragraph{Results.}
The trajectory of activation moments throughout training, as illustrated in Figure~\ref{app-fig4}, indicates that all three algorithms maintain relatively stable activation patterns without catastrophic collapse when averaged across adapter modules. This stability may be partially attributed to the architectural properties of the transformer, such as the normalizing effect of LayerNorm, or the smoothing inherent in averaging across layers. However, a distinct behavior is observable in SFed-LoRA (green curves): the higher ranks ($r=512, 2048$) exhibit a more dynamic evolution, with both mean and variance tending to level off later in the training process compared to the baselines. This delayed saturation suggests that SFed-LoRA facilitates sustained feature learning over a longer period, a consequence of the proposed rank stabilization that is particularly beneficial for high-rank training.
\end{document}